%% file: StagedBandits.tex
\documentclass{article}

\usepackage{microtype}
\usepackage{graphicx}
\usepackage{subfigure}
\usepackage{booktabs} 

\usepackage{hyperref}

\usepackage[accepted]{icml2018}

\usepackage{amsmath}
\usepackage{amsthm}
\usepackage{amssymb}
\usepackage{enumerate}
\usepackage{color}
\newtheorem{theorem}{Theorem}
\newtheorem{corollary}{Corollary}
\newtheorem{lemma}{Lemma}
\newtheorem{assumption}{Assumption}

\newtheorem{example-non*}{Example}

\newtheorem{remark}{Remark}

\allowdisplaybreaks 

\usepackage{mathtools}

\newcommand{\expect}{\operatorname{E}\expectarg}
\DeclarePairedDelimiterX{\expectarg}[1]{[}{]}{%
  \ifnum\currentgrouptype=16 \else\begingroup\fi
  \activatebar#1
  \ifnum\currentgrouptype=16 \else\endgroup\fi
}

\newcommand{\bs}{\boldsymbol}
\newcommand{\mr}{\mathrm}

\newcommand{\innermid}{\nonscript\;\delimsize\vert\nonscript\;}
\newcommand{\activatebar}{%
  \begingroup\lccode`\~=`\|
  \lowercase{\endgroup\let~}\innermid 
  \mathcode`|=\string"8000
}

\DeclareMathOperator*{\argmax}{arg\,max}

\ifodd 0
\newcommand{\revn}[1]{{\color{blue}#1}}
\else
\newcommand{\revn}[1]{#1}
\fi

\ifodd 0
\newcommand{\com}[1]{{\color{red}#1}}
\else
\newcommand{\com}[1]{}
\fi

\ifodd 0
\newcommand{\cnew}[1]{{\color{red}Comment: #1}}
\else
\newcommand{\cnew}[1]{}
\fi

\ifodd 1
\newcommand{\add}[1]{{\color{blue}#1}}
\else
\newcommand{\add}[1]{}
\fi

\ifodd 1
\newcommand{\rev}[1]{{\color{blue}#1}}
\else
\newcommand{\rev}[1]{#1}
\fi

\ifodd 1

\else

\fi

\ifodd 0
\newcommand{\ck}[1]{{\color{red}#1}}
\else
\newcommand{\ck}[1]{#1}
\fi

\ifodd 0
\newcommand{\ask}[1]{{\color{red}COM: #1}}
\else
\newcommand{\ask}[1]{}
\fi

\icmltitlerunning{Episodic Multi-armed Bandits}

\begin{document}
	
	\twocolumn[
	\icmltitle{Episodic Multi-armed Bandits}
	
	
	
	
	\begin{icmlauthorlist}
		\icmlauthor{Cem Tekin}{to}
		\icmlauthor{Mihaela van der Schaar}{goo}
	\end{icmlauthorlist}
	
	\icmlaffiliation{to}{Department of Electrical and Electronics Engineering, Bilkent University, Turkey}
	\icmlaffiliation{goo}{Department of Engineering, University of Oxford, United Kingdom}
	
	\icmlcorrespondingauthor{Cem Tekin}{cemtekin@ee.bilkent.edu.tr}
	\icmlcorrespondingauthor{Mihaela van der Schaar}{mihaela.vanderschaar@oxford-man.ox.ac.uk}
	
	\icmlkeywords{Episodic multi-armed bandits, online learning, reinforcement learning, submodularity}
	
	\vskip 0.3in
	]
	
	
	
	\printAffiliationsAndNotice{}  

\begin{abstract} 
We introduce a new class of reinforcement learning methods referred to as {\em episodic multi-armed bandits} (eMAB). In eMAB the learner
proceeds in {\em episodes}, each composed of several {\em steps}, in which it chooses an action and observes a feedback signal. Moreover, in each step, it can take a special action, called the $stop$ action, that ends the current episode. After the $stop$ action is taken, the learner collects a terminal reward, and observes the costs and terminal rewards associated with each step of the episode.
The goal of the learner is to maximize its cumulative gain (i.e., the terminal reward minus costs) over all episodes by learning to choose the best sequence of actions based on the feedback. First, we define an {\em oracle} benchmark, which sequentially selects the actions that maximize the expected immediate gain. 
Then, we propose our online learning algorithm, named {\em FeedBack Adaptive Learning} (FeedBAL), and prove that its regret with respect to the benchmark is bounded with high probability and increases logarithmically in expectation. Moreover, the regret only has polynomial dependence on the number of steps, actions and states.
eMAB can be used to model applications that involve humans in the loop, ranging from personalized medical screening to personalized web-based education, where sequences of actions are taken in each episode, and optimal behavior requires adapting the chosen actions based on the feedback.
\end{abstract} 

\section{Introduction}\label{sec:intro} 
 \input{intro}

\section{Problem Formulation} \label{sec:Formalism_Algorithm_Analysis} 
\input{probform}

\section{A Learning Algorithm for eMAB} \label{sec:percel} 

\input{algorithm}

\section{Performance Bounds for FeedBAL} \label{sec:regret}

\input{regret}

\section{Illustrative Example}\label{sec:exp}
\input{experiments2}

\section{Related Work}\label{sec:related}
\input{related}


\section{Conclusion}\label{sec:conc}
We proposed a new class of online learning methods called eMAB. Although the number of possible sequences of actions increases exponentially with the length of the episode, we proved that an efficient online learning algorithm which has expected regret that grows polynomially in the number of steps and states, and logarithmically in the number of episodes exists. This algorithm enjoys high probability confidence bounds on the expected gain of selected actions, and its regret is shown to be bounded with high probability.

\bibliography{OSA}
\bibliographystyle{icml2018}

\section*{APPENDICES}

\section{Approximate Optimality of the Benchmark for Adaptive Monotone Submodular eMAB}

In \citet{gabillon2013adaptive}, the adaptive submodular function to be maximized is given as $h: 2^{\cal A} \times \{-1,1\}^{S_{ {\cal A}}} \rightarrow \mathbb{R}$, where $2^{\cal A}$ denotes the power set of ${\cal A}$. 
The feedback observed after selecting an action is the state of that action.
Based on this, the set of observations is defined as ${\cal Y} := \{-1,0,1\}^{S_{ {\cal A}}}$. For an observation vector $\bs{y} \in {\cal Y}$, $y_a = 0$ implies that action $a$ is not selected, and hence, its state is not observed, while $y_a = i$, $i \in \{-1,1\}$ implies that action $a$ is selected and its state is observed as $i$. Let $\text{dom}(\bs{y})$ denote the set and $l(\bs{y})$ denote the number of actions whose states are observed according to observation vector $\bs{y}$.
They define the greedy policy for maximizing $h$ as $\pi^g$, such that given an observation vector $\bs{y}$, it selects the action
\begin{align}
&\pi^g(\bs{y}) = \argmax_{a \in {\cal A} - \text{dom}(\bs{y}) } \notag \\
&\mr{E}_{ \bs{s} | \bs{y} } [ h( \text{dom}(\bs{y}) \cup \{ a \} , \bs{s} ) 
- h( \text{dom}(\bs{y}) , \bs{s} )   ]  \label{eqn:golovin}
\end{align}
where the expectation is taken over the conditional distribution of $\bs{s}$ given $\bs{y}$. 
By linearity of conditional expectation \eqref{eqn:golovin} can be re-written as
\begin{align}
\mr{E}_{ \bs{s} | \bs{y} } [ h( \text{dom}(\bs{y}) \cup \{ a \} , \bs{s} ) ]
- \mr{E}_{ \bs{s} | \bs{y} } [ h( \text{dom}(\bs{y}) , \bs{s} )  ]  .    \notag
\end{align}
Note that the second term in the above equation does not depend on the choice of $a \in {\cal A} - \text{dom}(\bs{y}) $. Hence, $\pi^g$ can equivalently be defined as 
\begin{align}
&\pi^g(\bs{y}) = \argmax_{a \in {\cal A} - \text{dom}(\bs{y}) } 
\mr{E}_{ \bs{s} | \bs{y} } [ h( \text{dom}(\bs{y}) \cup \{ a \} , \bs{s} ) 
] . \label{eqn:golovin2}      
\end{align}

For a given feedback sequence $\bs{f}$, let $\bs{y}(\bs{f})$ be the observation vector that corresponds to $\bs{f}$. If $\bs{f}$ includes the feedback for action $a$, then $y_a(\bs{f})$ corresponds to this feedback, which is in $\{-1,1\}$. Otherwise, $y_a(\bs{f}) = 0$. Also, for an observation vector $\bs{y}$, let $\bs{s}(\bs{y})$ denote the states of actions in $\text{dom}(\bs{y}) \subset {\cal A}$. 
It is natural to assume in the setting of \citet{gabillon2013adaptive} that $h( \text{dom}(\bs{y}) \cup \{ a \} , \bs{s} ) $ only depends on the states of the actions in $\text{dom}(\bs{y}) \cup \{ a \}$. In \citet{gabillon2013adaptive}, an example of this is given for the maximum coverage problem. Moreover, it is assumed that the state of each action is drawn independently of the other actions. When the assumptions above hold, \eqref{eqn:golovin2} becomes
\begin{align}
&\pi^g(\bs{y}) = \argmax_{a \in {\cal A} - \text{dom}(\bs{y}) } 
\mr{E}_{ s_a } [ h( \text{dom}(\bs{y}) \cup \{ a \} , (\bs{s}(\bs{y}),  s_a ) ) ] . \label{eqn:golovin3}      
\end{align}

Let $t = S_{\text{dom}(\bs{y})} +1$ and ${\cal A}_t = \text{dom}(\bs{y})$ be the set of actions selected in the first $t$ steps. The above definition is equivalent to our benchmark if we define the state as the pair $(\text{dom}(\bs{y}), \bs{s}(\bs{y}))$.  Then, the ex-ante terminal reward of action $a \in {\cal A} - \text{dom}(\bs{y})$ becomes
\begin{align}
y_{t, (\text{dom}(\bs{y}), \bs{s}(\bs{y})), a} = E_{ s_{a} } 
[ r_{t+1, (\text{dom}(\bs{y}) \cup \{ a \}, (\bs{s}( \bs{y}), s_a)  )  } ]      \notag
\end{align}
where
\begin{align}
r_{S_{{\cal E}}+1, ({\cal E}, \bs{s}( {\cal E} )) } = h( {\cal E}, \bs{s}( {\cal E} ) ) .   \notag
\end{align}

It is shown in \citet{golovin2010adaptive} that the greedy policy is guaranteed to obtain at least $1-1/e$ of the expected reward of the optimal policy. Now consider our benchmark in this setting. Since it is known that $c_{t,x,a} = 0$, our benchmark will only stop after all actions in ${\cal A}$ are selected once. Therefore, our benchmark is $1-1/e$ approximately optimal for this special case. 

\section{Proof of Lemma 1} 

Fix any step-state-action triplet $(t,x,a)$. 
Let 
\begin{align}
{\cal E}_{\text{conf}}(t,x,a) :=  \left\{ | \hat{g}^{\rho}_{t,x,a} - g_{t,x,a} | \leq c_{t,x,a} ~ \forall \rho \geq 1  \right\}      \notag
\end{align}
By replacing $\delta$ term in \eqref{eqn:stopbound} and \eqref{eqn:sbound5} given in Appendix \ref{app:tripletbound} with $\delta/( l_{\max} S_{ {\cal X} } S_{ \bar{{\cal A}} } )$, we get
$\Pr( {\cal E}_{\text{conf}}(t,x,a) ) \geq 1 - \delta/( l_{\max} S_{ {\cal X} } S_{ \bar{{\cal A}} })$ (details can be found in Appendix \ref{app:tripletbound}). This implies that $\Pr( {\cal E}^c_{\text{conf}}(t,x,a) ) \leq \delta/( l_{\max} S_{ {\cal X} } S_{ \bar{{\cal A}} } )$ for all $t \in [l_{\max}]$, $x \in {\cal X}$ and $a \in \bar{ {\cal A} }$. 
Using a union bound, we get 
\begin{align}
\Pr({\cal E}^c_{\text{conf}}) 
&= \Pr \left( \bigcup_{t \in [l_{\max}]} \bigcup_{x \in {\cal X}} \bigcup_{a \in \bar{{\cal A}}}  {\cal E}^c_{\text{conf}}(t,x,a) \right)       \notag \\
&\leq \sum_{t \in [l_{\max}]} \sum_{x \in {\cal X}} \sum_{a \in \bar{{\cal A}}} \Pr( {\cal E}^c_{\text{conf}}(t,x,a) ) \notag \\
& \leq \delta. \notag 
\end{align}

\section{Proof of Lemma 2} 

For $\rho=1$, the result is trivial. For $\rho > 1$, the proof proceeds in a way that is similar to the proof of Lemma 6 in \citet{abbasi2011improved}. First, assume that action $a \in \bar{{\cal A}}$ is selected in step $t \in [T_{\rho}]$ of episode $\rho$ when the state is $x$.
Since 
\begin{align}
\hat{g}^{\rho}_{t,x,a} &\in [ g_{t,x,a} - \text{conf}^{\rho}_{t,x,a},  g_{t,x,a} + \text{conf}^{\rho}_{t,x,a}  ] \notag \\ 
\hat{g}^{\rho}_{t,x,a^*} &\in [ g^*_{t,x} - \text{conf}^{\rho}_{t,x,a^*},  g^*_{t,x} + \text{conf}^{\rho}_{t,x,a^*}  ], ~~ a^* \in {\cal O}_{t,x} \notag
\end{align}
on event ${\cal E}_{\text{conf}}$, using
\begin{align}
\hat{g}^{\rho}_{t,x,a} + \text{conf}^{\rho}_{t,x,a} & \geq  g^*_{t,x}  \label{eqn:confi1} \\
\hat{g}^{\rho}_{t,x,a}  & \leq g_{t,x,a} + \text{conf}^{\rho}_{t,x,a} \label{eqn:confi2}
\end{align}
and the definition of $\Delta_{t,x,a}$, we obtain $\text{conf}^{\rho}_{t,x,a} \geq \Delta_{t,x,a} / 2$. Substituting the values in \revn{Equations 3 and 4 of the manuscript} into $\text{conf}^{\rho}_{t,x,a}$ and using the fact that $(z^2 - 1)/(z+1) \leq z^2/(z+1)$ for positive integers $z$, we get for $a \in {\cal A}$
\begin{align}
\frac{ (N^{\rho}_{t,x,a})^2 - 1 }{ N^{\rho}_{t,x,a} +1  } 
\leq \frac{4}{\Delta^2_{t,x,a}}
\left( 4 \sigma^2 \log \left( \frac{ K (1 + N^{\rho}_{t,x,a})^{1/2} }{\delta} \right)   \right) \label{eqn:trickterm}
\end{align}

Now, assume that the $s := stop$ action is selected in step $t =T_{\rho}$ of episode $\rho$ when the state is $x$. Let 
\begin{align}
\overline{\text{conf}}^{\rho}_{t,x,s} =  \sqrt{ \frac{(1 + N^{\rho}_{t,x,s})}{ (N^{\rho}_{t,x,s} )^2} 
	\left( 4 \sigma^2 \log \left( \frac{ K (1 + N^{\rho}_{t,x,s})^{1/2} }{\delta} \right)   \right) }  .  \notag
\end{align}
Since $N^{\rho}_{t,x,s} \leq N^{\rho}_{t,x}$, we have $\text{conf}^{\rho}_{t,x,s} \leq \overline{\text{conf}}^{\rho}_{t,x,s}$, which implies that on event ${\cal E}_{\text{conf}}$
\begin{align}
\hat{g}^{\rho}_{t,x,a} &\in [ g_{t,x,a} - \overline{\text{conf}}^{\rho}_{t,x,s},  g_{t,x,a} + \overline{\text{conf}}^{\rho}_{t,x,s} ]      \notag \\
\hat{g}^{\rho}_{t,x,a^*} &\in [ g^*_{t,x} - \text{conf}^{\rho}_{t,x,a^*},  g^*_{t,x} + \text{conf}^{\rho}_{t,x,a^*}  ], ~ a^* \in {\cal O}_{t,x} . \notag
\end{align}
Using
\begin{align}
\hat{g}^{\rho}_{t,x,s} + \overline{\text{conf}^{\rho}_{t,x,s}} & \geq  g^*_{t,x}  \notag \\
\hat{g}^{\rho}_{t,x,s}  & \leq g_{t,x,s} + \overline{\text{conf}^{\rho}_{t,x,a}} \notag
\end{align}
and the definition of $\Delta_{t,x,a}$, we obtain $\overline{\text{conf}^{\rho}_{t,x,s}} \geq \Delta_{t,x,s} / 2$. This implies that \eqref{eqn:trickterm} also holds for $a = stop$.

Next, we use a lemma from \citet{antos2010active} to bound \eqref{eqn:trickterm}, which also given in Appendix \ref{app:antos}. 
From \eqref{eqn:trickterm} we obtain
\begin{align}
N^{\rho}_{t,x,a} &\leq 1 + \frac{16\sigma^2}{\Delta^2_{t,x,a}} \log \left( \frac{K}{\delta} \right)  \notag \\
&+ \frac{8\sigma^2}{\Delta^2_{t,x,a}} \log (1 + N^{\rho}_{t,x,a} )   \label{eqn:trickterm2}
\end{align}
Since $1 + N^{\rho}_{t,x,a}   \geq 1$, we substitute $a= \Delta^2_{t,x,a} / (16 \sigma^2 )$ and $b = \log (   16 \sigma^2 / \Delta^2_{t,x,a})$ in Appendix \ref{app:antos} to get the bound
\begin{align}
\log (1 + N^{\rho}_{t,x,a} ) \leq a (1 + N^{\rho}_{t,x,a} ) + b .  \notag
\end{align}
The result is obtained by substituting this into \eqref{eqn:trickterm2}.

\section{Proof of Corollary 1} 

The result follows by a simple application of \eqref{eqn:confi1} and \eqref{eqn:confi2} on event ${\cal E}_{\text{conf}}$.

\section{Proof of Theorem 1} 

The proof directly follows by summing the result of \revn{Lemma 2} among all step-state-action triplets $(t,x,a)$.

\section{Proof of Theorem 2} 

Consider \revn{Theorem 1}. With probability $\delta$, the regret is bounded above by $n \Omega_{\max}$. With probability $1-\delta$, the regret is bounded by the theorem's main statement. The proof follows from the law of total expectation.

\section{A Confidence Bound for Step-State-Action Triplet} \label{app:tripletbound}
\input{tripletbound}

\section{Lemma 8 of \citet{antos2010active}} \label{app:antos}
\input{antos}

\section{Additional Numerical Results}

In Section 5 of the paper we present the results for FeedBAL by setting $\sigma^2 = 0.2$ since $\sigma^2_c$ and $\sigma^2_r$ are taken as $0.1$. This term, which comes from the $\sigma$-sub-Gaussian noise process assumption appears in the confidence numbers of FeedBAL. Here, we give regret results for FeedBAL when it takes as input $\sigma^2$ values from the set $\{0.05, 0.02, 0.2, 0.4\}$. The results given in Figure 1 show that the regret of FeedBAL is the smallest for $\sigma^2 = 0.05$ and the largest for $\sigma^2 = 0.02$. This shows that shrinking the confidence intervals beyond the theoretical limit suggested in \citet{abbasi2011improved} may result in a sharp increase in the regret. 
On the other hand, the regret of FeedBAL for $\sigma^2 = 0.4$ is larger than the regret for $\sigma^2 = 0.2$ but smaller than the regret for $\sigma^2 = 0.02$, which is expected since a larger confidence number implies a greater number of explorations. 

\begin{figure}[h!]
	\centering
	\includegraphics[width=0.9\columnwidth]{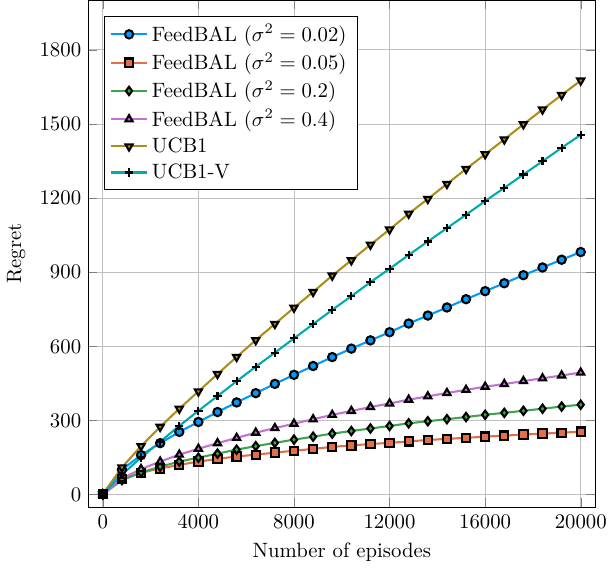}
	\vspace{-0.2in}
	\caption{Regrets of FeedBAL, UCB1 and UCB1-V as a function of the number of episodes.}
	\label{fig_result}
\end{figure}

\section*{Acknowledgement}

We thank Hamza Yusuf \c{C}ak{\i}r for his help in simulations.

\end{document}

%% file: intro.tex
Many applications involving sequential decision making under uncertainty
can be formalized as multi-armed bandits (MAB): clinical trials \cite{lai1}, recommender systems and
web advertising \cite{slivkins2009contextual,li2010contextual} etc.
A common assumption in all these problems is that each decision step
involves taking an action after which a reward is observed. Although MAB extensions
also allow for settings in which the rewards are missing, delayed
or erroneous or multiple actions are taken simultaneously, 
in numerous applications such as
humanoid robot locomotion \cite{nassour2013qualitative}, online education \cite{IT6}
and healthcare \cite{schaefer2004modeling}, each decision step involves
taking multiple actions whose reward is only
revealed after the entire action sequence is completed and a decision
is made to stop the action sequence and (possibly) take a final action.

For instance, in personalized online education, a sequence of materials can be used
to teach or remind students the key concepts of a course subject.
While the final exam is used as a benchmark to evaluate the overall
effectiveness of the given sequence of teaching materials, a sequence
of intermediate feedbacks like students' performance on quizzes, homework
grades, etc., can be used to guide the teaching examples online. Similarly,
in personalized healthcare, a sequence of treatments is given to a
patient over a period of time. The overall effectiveness of the treatment
plan depends on the given treatments as well as their order \cite{schaefer2004modeling}.
Moreover, the patient can be monitored during the course of the treatment
which yields a sequence of feedbacks about the selected treatments,
while the final outcome is only available after the entire sequence
of treatments is completed.

In conclusion, in such sequential decision making problems the 
order of the taken actions matters. Moreover, the feedback
available after each taken action drives the action selection process.
We call online learning problems exhibiting the aforementioned properties
{\em episodic multi-armed bandits} (eMAB). In eMAB the learner proceeds
in episodes $\rho=1,2,\ldots$ composed of multiple steps, in which
the learner selects actions sequentially in steps, one after another,
with each action belonging to the action set ${\cal A}$. After each
taken action $a \in {\cal A}$, a feedback $f \in {\cal F}$ is observed
about the taken action. Based on all its previous observations in
that episode, the learner either decides to continue to the next step
by selecting another action or selecting a $stop$ action which ends
the current episode, yields a terminal reward, and starts the next episode. Hence, the number of
steps in each episode is a decision variable, and the terminal rewards and losses of the steps in an episode
are observed only after the $stop$ action is taken. 
The goal of the learner is to maximize its total expected gain (i.e., the terminal reward minus costs) over all episodes by learning to choose the best action sequence given the feedback.
An illustration that shows the order of steps, costs, terminal rewards and episodes
is given in Figure \ref{fig:mainfigure}.

Observing the terminal rewards of the previous steps is possible in many problems where actions correspond to revealing hidden features, and the $stop$ action corresponds to performing classification or detection using the features revealed so far. These include active sensing problems \cite{yu2009active}, and multi-view classification based on the observed features \cite{muslea2002active}.

The contributions are summarized as follows:
\begin{itemize}
\item We propose a new online learning model called eMAB, which covers other learning models including the online adaptive submodular maximization problem \cite{gabillon2013adaptive} as special cases and propose the {\em FeedBack Adaptive Learning} (FeedBAL) algorithm.
\item We compare FeedBAL with a benchmark that always chooses the myopic best action given the current feedback, and prove that it achieves $O(\log n)$ regret, where $n$ denotes the number of episodes. Moreover, the regret has polynomial dependence on the number of steps, actions and states.
\item We perform experiments on FeedBAL and compare its performance with existing methods. 
\end{itemize}

\begin{figure*}[htb]
\begin{center}
\includegraphics[scale=0.50]{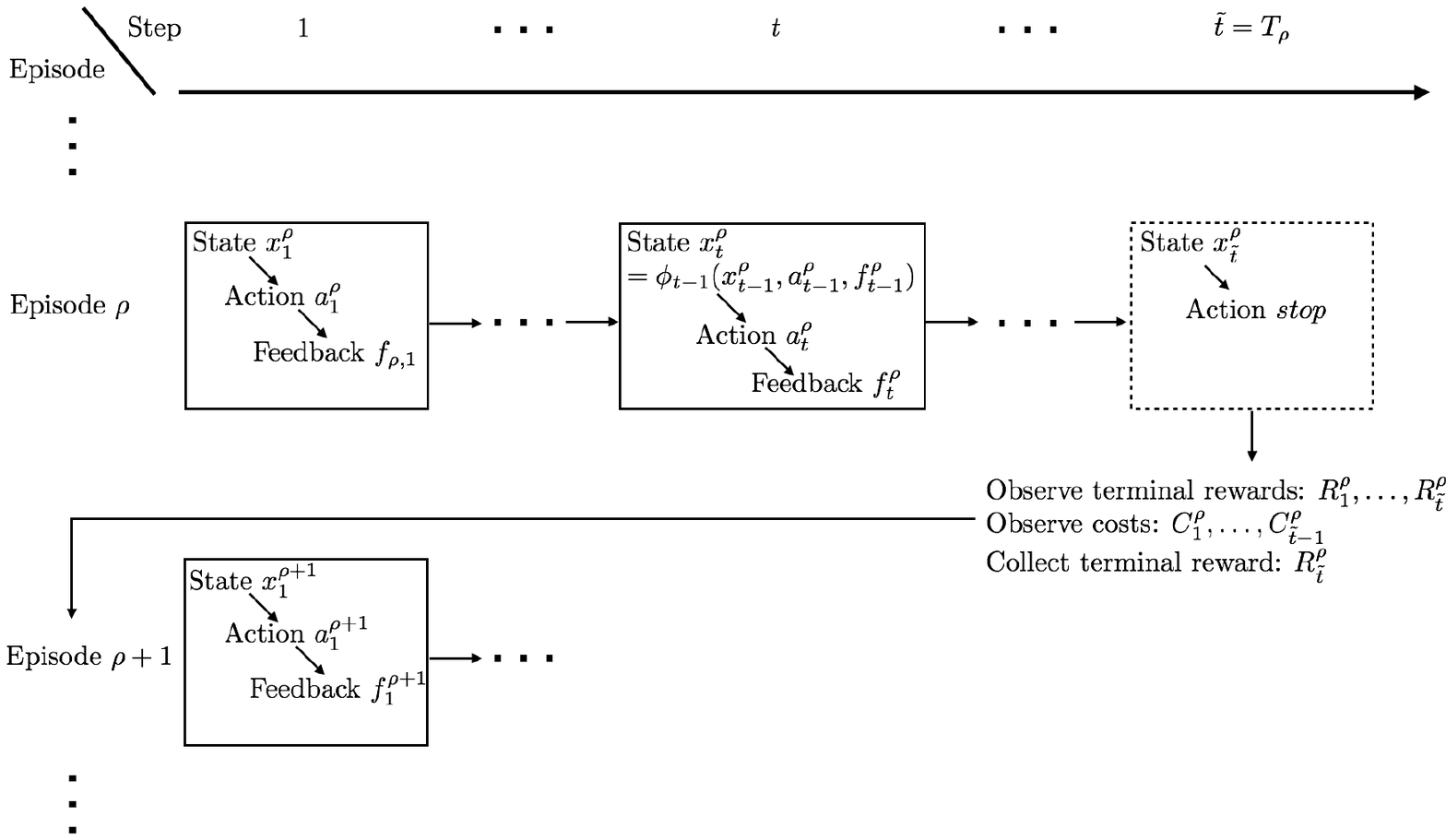}
\protect\caption{$x^{\rho}_{t}$ is the state observed, 
$a^{\rho}_{t}$ is the action selected and $f^{\rho}_{t}$ is the feedback observed in the step $t$ of episode $\rho$. $C^{\rho}_t$ is the cost of selecting action $a^{\rho}_t$ and $R^{\rho}_t$ is the terminal reward in step $t$ of episode $\rho$. $T_{\rho}$ is the step in which the learner selects the $stop$ action after which the costs and terminal rewards are revealed.}
\label{fig:mainfigure}
\vspace{-0.1in}
\end{center}
\end{figure*}

Rest of the paper is organized as follows. 
Problem formulation, and the definitions of the benchmark and the regret
are given in Section \ref{sec:Formalism_Algorithm_Analysis}.
The learning algorithm is introduced in Section \ref{sec:percel}. 
Regret analysis of this algorithm is provided in Section \ref{sec:regret}.
Illustrative results are given in Section \ref{sec:exp}. Related work and concluding remarks are given in Sections \ref{sec:related} and \ref{sec:conc} respectively. \revn{All proofs are given in the supplemental document.}

%% file: probform.tex
\subsection{Notation}\label{sec:notation}

Sets are denoted by calligraphic letters, vectors are denoted by boldface letters and random variables are denoted by capital letters. For a set ${\cal E}$, $S_{{\cal E}} := |{\cal E}|$, where $|\cdot|$ denotes the cardinality.
The set of positive integers up to integer $t$ is denoted by $[t]$. $\mr{E}_{p}[\cdot]$ denotes the expectation with respect to probability distribution $p$. $\mr{I}( {\cal E} )$ denotes the indicator function of event ${\cal E}$ which is one if ${\cal E}$ is true and $0$ otherwise. For a set ${\cal E}$, $\Delta({\cal E} )$ denotes the set of probability distributions over ${\cal E}$.  All inequalities that involve random variables hold with probability one. 

\subsection{Problem Description}\label{sec:description}

The learner proceeds in episodes indexed by $\rho$. Each episode is composed of multiple steps indexed by $t$. Each step corresponds to a decision epoch in which the learner can choose an action from a finite set of actions denoted by $\bar{{\cal A}}$. There are two types of actions in $\bar{{\cal A}}$: (i) continuation actions which move the learner to the next step and allow it to acquire more information (feedback), (ii) a terminal action (also named as the $stop$ action) which ends the current episode and yields a terminal reward. 

The set of continuation actions is denoted by ${\cal A}$. The maximum number of steps in an episode is $l_{\max} < \infty$, which implies that the $stop$ action must be selected in at most $l_{\max}$ steps. After an action $a \in {\cal A}$ is selected in a step $t$, the learner observes a feedback $f \in {\cal F}$ before moving to the next step, where ${\cal F}$ denotes the set of all feedbacks.

Let $\bs{a}_{[t]} := (a_1, a_2, \ldots, a_t)$ denote a length $t$ sequence of continuation actions
and $\bs{f}_{[t]} := (f_1, f_2, \ldots, f_t)$ denote a length $t$ sequence of feedbacks. 
Let ${\cal A}^t := \prod_{i=1}^{t} {\cal A}$ denote the set of length $t$ sequences of continuation actions and ${\cal F}^t := \prod_{i=1}^{t} {\cal F}$ denote the set of length $t$ sequences of feedbacks. Set of all continuation action sequences is denoted by ${\cal A}^{\text{all}} := \bigcup_{t=1}^{l_{\max}-1} {\cal A}^t$ and the set of all feedback sequences is denoted by ${\cal F}^{\text{all}} := \bigcup_{t=1}^{l_{\max}-1} {\cal F}^t$.
At each step, the system is in one of the finitely many states, where the set of states is denoted by ${\cal X}$.

When action $a$ is chosen in step $t$, the feedback it generates depends on the state of the system in that step. Specifically, we assume that $f_t \sim p_{t,x,a} \in \Delta({\cal F})$, where $p_{t,x,a}$ denotes the probability distribution of the feedback given the step-state-action triplet $(t,x,a)$. 
Let $\phi_t: {\cal X} \times {\cal A} \times {\cal F} \rightarrow {\cal X}$ be the {\em state transition function} which encodes every state-action-feedback triplet to one of the states in ${\cal X}$.
Since the feedback is random, the next state is not a deterministic function of the previous state. Moreover, the state transition probabilities are step dependent.\footnote{Hence, our definition of state is more general than the definition of state used in reinforcement learning in MDPs \cite{tewari2008, auer2009near}, which is assumed to be time homogeneous.}

The expected cost of action $a$ in step $t$ when the state is $x$ is given by $c_{t,x,a} \in [0, c_{\max}]$ and
the expected terminal reward in step $t$ when the state is $x$ is given by $r_{t,x} \in [0, r_{\max}]$. The {\em ex-ante} terminal reward of the triplet $(t,x,a)$ is defined as 
$y_{t,x,a}
:= \mr{E} [r_{t+1, \phi_t(x,a,f_t)} ]$  
which gives the expected terminal reward of stopping at step $t+1$ after choosing action $a$ in step $t$ and before observing the feedback $f_t$.
For the $stop$ action the cost is always zero and
$y_{t,x,stop} = r_{t,x}, ~ \forall t \in [l_{\max}], ~\forall x \in {\cal X}$. 
The {\em gain} of an action $a \in \bar{{\cal A}}$ in step $t$ when the state is $x$ is defined as 
$g_{t,x,a} := y_{t,x,a} - c_{t,x,a}$.    

At each episode $\rho$, the learner chooses a sequence of actions $\bs{a}^{\rho} := (a^\rho_1, \ldots, a^\rho_{T_\rho})$, observes a sequence of feedbacks $\bs{f}^{\rho} := (f^\rho_1, \ldots, f^\rho_{T_\rho-1})$ and encounters a sequence of states $\bs{x}^{\rho} := (x^\rho_1, \ldots, x^\rho_{T_\rho})$, where $T_\rho$ denotes the step in which the $stop$ action is taken. Since no feedback is present in the first step, we set $x^\rho_1 = 0$.
After the $stop$ action is taken, the learner observes costs of the selected actions $C^\rho_{t} = c_{t,x^\rho_t,a^\rho_t} + \eta^\rho_t$ for $t \in [T_\rho-1]$ and the terminal rewards $R^\rho_{t} = r_{t,x^\rho_t} + \kappa^\rho_t$ for $t \in [T_\rho]$, where $\eta^\rho_t$ and $\kappa^{\rho}_t$ are independent $\sigma$-sub-Gaussian random variables that
are also independent from $x^{\rho}_{1:t}, a^{\rho}_{1:t}, f^{\rho}_{1:t}, \kappa^{\rho}_{1:t-1}, \eta^{\rho}_{1:t-1} $, i.e., $\forall \lambda \in \mathbb{R}$ and $\theta^{\rho}_t \in \{ \eta^\rho_t, \kappa^{\rho}_t \}$,
$\expect{ e^{\lambda \theta^{\rho}_t} } \leq \exp \left( \frac{\lambda^2 \sigma^2}{2} \right)$. 
When the episode is clear from the context, we will drop the superscripts from the expressions above. 

We assume that the learner knows the state transition function and can compute the state of the system at any step by using the actions taken and feedbacks observed in the previous steps. The learner does not know the feedback, cost and terminal reward distributions. The goal of the learner is to maximize its cumulative gain over the episodes by repeated interaction with the system. 

An important application of eMAB is medical screening, where screening high risk patients using multiple modalities may improve the chance of early detection and longer survival \cite{berg2008combined}. In this application, the actions correspond to medical screening tests such as mammogram, ultrasound and MRI, and feedbacks correspond to the BI-RADS scores from the administered tests. Based on this, the state can represent the likelihood of having cancer, which will change after each new screening test. The terminal reward can represent the reward of detection, missed detection or false alarm that results from the final assessment made after the $stop$ action is taken. Finally, the costs can represent the financial costs of administering the screening tests. 

 \vspace{-0.1in}
\subsection{The Benchmark} \label{sec:thebenchmark}
 
Since the number of possible action and feedback sequences is exponential in $l_{\max}$, it is very inefficient to learn the best action sequence by separately estimating the expected gain of each action sequence $\bs{a} \in {\cal A}^{\text{all}}$. In this section we propose a benchmark (given in Algorithm \ref{fig:BS}) whose action selection strategy can be learned quickly.

\begin{algorithm}[h]
{\fontsize{9}{12}\selectfont
\caption{The Benchmark} \label{fig:BS}
\begin{algorithmic}[1]

\REQUIRE ${\cal A}$, ${\cal X}$, $l_{\max}$

Initialize: $\rho = 1$

\WHILE{$\rho \geq 1$}
\STATE{$t=1$, $x_1 = 0$}
\WHILE{$t \in [l_{\max}]$}
\IF{$ (stop \in \argmax_{ a \in \bar{{\cal A}} }  g_{t,x_t,a}) ~ || ~  (t = l_{\max})$}
\STATE{$a^*_{t} = stop$, $T^*_{\rho} = t$ //BREAK}
\ELSE
\STATE{Select $a^*_{t}$ from $\argmax_{ a \in {\cal A} }  g_{t,x_t,a} $}
\ENDIF
\STATE{Observe feedback $f_t$}
\STATE{Set $x_{t+1} = \phi_t(x_t, a^*_t, f_t )$}
\STATE{$t = t+1$}
\ENDWHILE
\STATE{Observe the costs $C^*_t$, $t \in [T^*_\rho-1]$}
\STATE{Collect terminal reward $R_{T^*_{\rho}}$}
\STATE{$\rho = \rho+1$}
\ENDWHILE
\end{algorithmic}
}
\end{algorithm}

The benchmark\footnote{This benchmark is similar to the best first search algorithms for
graphs \cite{vempaty1991depth}. Moreover, it is shown that this benchmark
is approximately optimal for problems exhibiting adaptive submodularity
\cite{golovin2010adaptive}.} incrementally selects the next action based on the past sequence of feedbacks and actions. If the $stop$ action is not taken up to step $t$, the benchmark selects its action in step $t$ according to the following rule: Assume that the state in step $t$ is $x$. 
If $g_{t,x,stop} \geq g_{t,x,a}$ for all $a \in {\cal A}$ (which implies that $r_{t,x} \geq y_{t,x,a} - c_{t,x,a}$ for all $a \in {\cal A}$), then the benchmark selects the $stop$ action in step $t$.
Otherwise, it decides to continue for one more step by selecting one of the actions $a \in {\cal A}$ which maximizes $g_{t,x,a}$.

Let $\boldsymbol{a}^{*\rho} := (a^{*\rho}_{1}, \ldots, a^{*\rho}_{T^*_{\rho}} )$ be the action sequence selected, $\boldsymbol{x}^{*\rho} := (x^{*\rho}_{1}, \ldots, x^{*\rho}_{T^*_{\rho}} )$ be the state sequence, $\bs{C}^{*\rho} := (C^{*\rho}_1, \ldots, C^{*\rho}_{T^*_{\rho}-1})$ be the cost sequence observed,
and $R^{*\rho}_{T^{*\rho}}$ be the terminal reward collected by the benchmark in episode $\rho$, where $T^*_{\rho}$ is the step in which the $stop$ action is selected.
The cumulative expected {\em gain}, i.e., the expected terminal reward minus costs, of the benchmark in the first $n$ episodes is equal to 
\begin{align}
RW_{\textrm{B}}(n) := 
\expect*{ \sum_{\rho=1}^{n}  
\left( R^{*\rho}_{ T^*_{\rho} }  - \sum_{t=1}^{ T^*_{\rho}-1 } C^{*\rho}_t \right) } .  \notag
\end{align}

 \vspace{-0.1in}

Next, we evaluate the performance of the benchmark under two important special cases. Another important case, in which the benchmark is the optimal policy is given in Section \ref{sec:exp}.

\textbf{Approximate optimality of the benchmark in adaptive monotone submodular eMAB:}
Assume that an action state $s_a \in \{-1, 1 \}$ is associated with each action $a \in {\cal A}$, and the joint action state vector $\bs{s} = \{  s_a \}_{a \in {\cal A}} \in \{-1,1\}^{S_{ {\cal A}}}$ is sampled independently from some fixed distribution at the beginning of each episode. 
Consider a special case of eMAB in which: (i) the state $x$ is defined as a pair that consists of the set of actions selected so far and their action states, (ii) action selection costs are set to zero, i.e., $c_{t,x,a} = 0$, (iii) $l_{\max} \leq S_{ \bar{{\cal A}} }$, (iv) if an action is selected in step $t$ it cannot be selected in the future steps, and (v) $r_{t,x}$ is an adaptive submodular function of $x$. These assumptions reduce our problem to the adaptive submodular maximization problem \cite{golovin2010adaptive, gabillon2013adaptive}, where our benchmark is $1-1/e$ approximately optimal (for details see the supplemental document). 

\comment{
In \cite{gabillon2013adaptive}, the adaptive submodular function to be maximized is given as $h: 2^{\cal A} \times \{-1,1\}^{S_{ {\cal A}}} \rightarrow \mathbb{R}$, where $2^{\cal A}$ denotes the power set of ${\cal A}$. 
The feedback observed after selecting an action is the state of that action.
Based on this, the set of observations is defined as ${\cal Y} := \{-1,0,1\}^{S_{ {\cal A}}}$. For an observation vector $\bs{y} \in {\cal Y}$, $y_a = 0$ implies that action $a$ is not selected, and hence, its state is not observed, while $y_a = i$, $i \in \{-1,1\}$ implies that action $a$ is selected and its state is observed as $i$. Let $\text{dom}(\bs{y})$ denote the set and $l(\bs{y})$ denote the number of actions whose states are observed according to observation vector $\bs{y}$.
They define the greedy policy for maximizing $h$ as $\pi^g$, such that given an observation vector $\bs{y}$, it selects the action
\begin{align}
&\pi^g(\bs{y}) = \argmax_{a \in {\cal A} - \text{dom}(\bs{y}) } \notag \\
&\mr{E}_{ \bs{s} | \bs{y} } [ h( \text{dom}(\bs{y}) \cup \{ a \} , \bs{s} ) 
- h( \text{dom}(\bs{y}) , \bs{s} )   ] . \label{eqn:golovin}
\end{align}
where the expectation is taken over the conditional distribution of $\bs{s}$ given $\bs{y}$. 
By linearity of conditional expectation \eqref{eqn:golovin} can be re-written as
\begin{align}
\mr{E}_{ \bs{s} | \bs{y} } [ h( \text{dom}(\bs{y}) \cup \{ a \} , \bs{s} ) ]
- \mr{E}_{ \bs{s} | \bs{y} } [ h( \text{dom}(\bs{y}) , \bs{s} )  ]  .    \notag
\end{align}
Note that the second term in the above equation does not depend on the choice of $a \in {\cal A} - \text{dom}(\bs{y}) $. Hence, $\pi^g$ can equivalently be defined as 
\begin{align}
&\pi^g(\bs{y}) = \argmax_{a \in {\cal A} - \text{dom}(\bs{y}) } 
\mr{E}_{ \bs{s} | \bs{y} } [ h( \text{dom}(\bs{y}) \cup \{ a \} , \bs{s} ) 
 ] . \label{eqn:golovin2}      
\end{align}

For a given feedback sequence $\bs{f}$, let $\bs{y}(\bs{f})$ be the observation vector that corresponds to $\bs{f}$. If $\bs{f}$ includes the feedback for action $a$, then $y_a(\bs{f})$ corresponds to this feedback, which is in $\{-1,1\}$. Otherwise, $y_a(\bs{f}) = 0$. Also, for an observation vector $\bs{y}$, let $\bs{s}(\bs{y})$ denote the states of actions in $\text{dom}(\bs{y}) \subset {\cal A}$. 
It is natural to assume in the setting of \cite{gabillon2013adaptive} that $h( \text{dom}(\bs{y}) \cup \{ a \} , \bs{s} ) $ only depends on the states of the actions in $\text{dom}(\bs{y}) \cup \{ a \}$. In \cite{gabillon2013adaptive}, an example of this is given for the maximum coverage problem. Moreover, it is assumed that the state of each action is drawn independently of the other actions. When the assumptions above hold \eqref{eqn:golovin2} becomes
\begin{align}
&\pi^g(\bs{y}) = \argmax_{a \in {\cal A} - \text{dom}(\bs{y}) } 
\mr{E}_{ s_a } [ h( \text{dom}(\bs{y}) \cup \{ a \} , (\bs{s}(\bs{y}),  s_a ) ) ] . \label{eqn:golovin3}      
\end{align}

Let $t = S_{\text{dom}(\bs{y})} +1$ and ${\cal A}_t = \text{dom}(\bs{y})$ be the set of actions selected in the first $t$ steps. The above definition is equivalent to our benchmark if we define the state as the pair $(\text{dom}(\bs{y}), \bs{s}(\bs{y}))$.  Then, the ex-ante terminal reward of action $a \in {\cal A} - \text{dom}(\bs{y})$ becomes
\begin{align}
y_{t, (\text{dom}(\bs{y}), \bs{s}(\bs{y})), a} = E_{ s_{a} } 
[ r_{t+1, (\text{dom}(\bs{y}) \cup \{ a \}, (\bs{s}( \bs{y}), s_a)  )  } ]      \notag
\end{align}
where
\begin{align}
r_{S_{{\cal E}}+1, ({\cal E}, \bs{s}( {\cal E} )) } = h( {\cal E}, \bs{s}( {\cal E} ) ) .   \notag
\end{align}

It is shown in \cite{golovin2010adaptive} that the greedy policy is guaranteed to obtain at least $1-1/e$ of the expected reward of the optimal policy. Now consider our benchmark in this setting. Since it is known that $c_{t,x,a} = 0$, our benchmark will only stop after all actions in ${\cal A}$ are selected once. Therefore, our benchmark is $1-1/e$ approximately optimal for this special case. 
}

\textbf{Optimality of the benchmark and its performance against the best fixed sequence of actions:}
Here, we show that the benchmark can perform much better than the best fixed action sequence that ends with the $stop$ action that is not adapted based on the observed feedbacks. 
For this, let $l_{\max}=3$, ${\cal X} = \{0,1,2\}$, ${\cal F} = \{0, 1\}$, ${\cal A} = \{ a_0, a_1 \}$, $c_{t,x,a}=1$, $\forall t \in [2]$, $\forall x \in {\cal X}$ and $\forall a \in {\cal A}$, $\phi_t(x,a,f) = x + f$, $r_{t,x} = t^2 \text{I}(x=\text{odd})$ and $p_{t,x,a} = \text{Ber}(q_{x,a})$, where $q_{x,a}$ is the parameter of the Bernoulli distribution. Assume that $q_{0,a_0} = q_{1,a_0} =1$ and $q_{0,a_1} = q_{1,a_1} =0$. Clearly, the terminal reward function does not exhibit diminishing returns property, and hence, is not adaptive submodular in this case.
The fixed action sequences in this case are $stop$, $(a_0,stop)$, $(a_1,stop)$, $(a_0,a_0,stop)$ and $(a_1,a_1,stop)$. It is easy to check that the best fixed action sequence is $(a_0, stop)$, whose cumulative gain is $3$. On the other hand, the benchmark will select the sequence $(a_0,a_1,stop)$, which yields a cumulative gain of $7$. Moreover, in this case, $(a_0,a_1,stop)$ is the optimal action sequence.

\vspace{-0.1in}
\subsection{Definition of the Regret} 

The (pseudo) regret of a learning algorithm which selects the action sequence $\boldsymbol{a}^\rho$ and observes the feedback sequence $\boldsymbol{f}^\rho$ in episode $\rho$ with respect to the benchmark in the first $n$ episodes is given by
\begin{align}
R(n) &:=     \left( \sum_{\rho=1}^n \left( r_{T^*_{\rho}, x^{*\rho}_{T^*_{\rho}} } - 
\sum_{t=1}^{T^*_{\rho} - 1} c_{t, x^{*\rho}_t, a^{*\rho}_t }  \right) \right) \notag \\ 
&- \left( \sum_{\rho=1}^n \left( r_{T_{\rho}, x^{\rho}_{T_{\rho}} } - 
\sum_{t=1}^{T_{\rho} - 1} c_{t, x^{\rho}_t, a^{\rho}_t } \right)   \right) . \label{eqn:pseudoregret}
\end{align}
When we take expectation of \eqref{eqn:pseudoregret} over all sources of randomness, we obtain the expected regret, which is equivalent to
\begin{align}
\hspace{-0.1in} \mathrm{E} [R(n)] &=
RW_{B}(n)  
-  \expect*{ \sum_{\rho=1}^{n} \left( R^{\rho}_{T_{\rho}}  - \sum_{t = 1}^{T_{\rho}-1} C^{\rho}_t \right) } . \label{eqn:regretdef}
\end{align}
Any algorithm whose expected regret increases at most sublinearly, i.e., $\mathrm{E} [R(n)] = O(n^\gamma)$, $0<\gamma<1$, in the number of episodes will converge in terms of the average reward to the average reward of the benchmark as $n \rightarrow \infty$. In the next section we propose an algorithm whose expected regret increases only logarithmically in the number of episodes and polynomially in the number of steps.

%% file: algorithm.tex
In this section we propose {\em Feedback Adaptive Learning} (FeedBAL) (pseudocode given in Figure \ref{fig:FAL}), which learns the sequence of actions to select based on the observed feedbacks to the actions taken in previous steps of an episode (as shown in Figure \ref{fig:mainfigure}). In order to minimize the regret given in (\ref{eqn:regretdef}), FeedBAL balances exploration and exploitation when selecting the actions. 

FeedBAL keeps the sample mean estimates $\hat{g}^{\rho}_{t,x,a}$ of the gains $g^{\rho}_{t,x,a}$ of the actions $a \in \bar{{\cal A}}$ and the sample mean estimates $\hat{r}^{\rho}_{t,x}$ of the terminal rewards $r^{\rho}_{t,x}$ for all step-state pairs $(t,x)$. Using the definition of the gain for the $stop$ action it sets
$\hat{g}^{\rho}_{t,x,stop} = \hat{r}^{\rho}_{t,x}$ for all $(t,x)$.
In addition to these, FeedBAL also keeps the following counters: 
$N^{\rho}_{t,x}$ which counts the number of times step-state pair $(t,x)$ is observed\footnote{We say that a step-state pair $(t,x)$ is observed in episode $\rho$ if the state is $x$ at step $t$ of episode $\rho$.} prior to episode $\rho$, and
$N^{\rho}_{t,x,a}$ which counts the number of times action $a \in \bar{{\cal A}}$ is selected after step-state pair $(t,x)$ is observed prior to episode $\rho$.

\begin{algorithm} [h]
{\fontsize{9}{12}\selectfont
\caption{FeedBack Adaptive Learning (FeedBAL)}\label{fig:FAL}
\begin{algorithmic}[1]

\REQUIRE ${\cal A}$, ${\cal X}$, $l_{\max}$, $\sigma$, $\delta$

Initialize counters: $N_{t,x}=0$, $N_{t,x,a}=0$, $\forall t \in [ l_{\max} ], ~ \forall x \in {\cal X}, ~ \forall a \in {\cal A}$, and $\rho=1$. 

Initialize estimates: $\hat{r}_{t,x} =0$, $\hat{g}_{t,x,a}=0$, $\forall t \in [ l_{\max} ], ~ \forall x \in {\cal X}, ~ \forall a \in {\cal A}$.

\WHILE{$\rho \geq 1$}
\STATE{$t=1$, $x_1 = 0$}
\WHILE{$t \in [ l_{\max} ]$}
 \STATE{Calculate UCBs: $u_{t,x_t,a} = \hat{g}_{t,x_t,a} + \text{conf}_{t,x_t,a}$, $\forall a \in \bar{{\cal A}}$, where $\text{conf}_{t,x_t,a}$ is given in \eqref{eqn:confterm1} and \eqref{eqn:confterm2} 
}
\IF{$(stop \in \argmax_{a \in \bar{{\cal A}}} u_{t,x_t,a}) ~ || ~ (t = l_{\max})$}
\STATE{$a_t = stop$, $T_{\rho} = t$ // BREAK}
\ELSE
\STATE{Select $a_t$ from $\argmax_{a \in {\cal A}} u_{t,x_t,a}$}
\ENDIF
\STATE{Observe feedback $f_t$}
\STATE{Set $x_{t+1} = \phi_t (x_t, a_t, f_t)$}
\STATE{$t = t+1$}
\ENDWHILE
\STATE{Observe the costs $C^{\rho}_t$, $t \in [T_\rho-1]$ and the terminal rewards $R^{\rho}_t$, $t \in [T_\rho]$}
\STATE{Collect terminal reward $R^{\rho}_{T_{\rho}}$}
\STATE{Update:\\
(i) $\hat{g}_{t,x,stop} = \hat{r}_{t,x}  =
\frac{N_{t,x} \hat{r}_{t,x} + R^{\rho}_t \mr{I}(x_t = x)}
{N_{t,x}+ \mr{I}(x_t = x)}$, for $t \in [ T_\rho ]$ and $x \in {\cal X}$ \\ (ii) $N_{t,x}  = N_{t,x} + \mr{I}(x_t = x)$
for $t \in [ T_\rho ]$ and $x \in {\cal X}$; \\(iii) $\hat{g}_{t,x,a} =
\frac{N_{t,x,a} \hat{g}_{t,x,a} + (R^{\rho}_{t+1} - C^{\rho}_t) \mr{I}(x_t = x, a_t =a) }
{N_{t,x,a}+\mr{I}(x_t = x, a_t =a) }$ for $t \in [T_\rho - 1]$, $x \in {\cal X}$ and $a \in {\cal A}$; \\(iv) $N_{t,x,a} = N_{t,x,a} + \mr{I}(x_t = x, a_t =a) $ for $t \in [T_\rho - 1]$, $x \in {\cal X}$ and $a \in {\cal A}$
}
\STATE{$\rho=\rho+1$}
\ENDWHILE
\end{algorithmic}
}
\end{algorithm}

Next, we explain the operation of FeedBAL. Consider step $t$ of episode $\rho$. If FeedBAL has not selected the $stop$ action yet, using its knowledge of the state $x^\rho_t$, it calculates the following upper confidence bounds (UCBs):
$u^{\rho}_{t,x^\rho_t,a} := \hat{g}^{\rho}_{t,x^\rho_t,a}+ \text{conf}^{\rho}_{t,x^\rho_t,a}$
for the actions in $\bar{{\cal A}}$, where $\text{conf}^{\rho}_{t,x^\rho_t,a}$ denotes the {\em confidence number} for the triplet $(t,x,a)$, which is given as
\begin{align}
& \text{conf}^{\rho}_{t,x^\rho_t,a} \notag \\
&\hspace{-0.1in} =  \hspace{-0.1in} \sqrt{ \frac{(1 + N^{\rho}_{t,x^\rho_t,a})}{ (N^{\rho}_{t,x^\rho_t,a} )^2} 
\left( 4 \sigma^2 \log \left( \frac{K (1 + N^{\rho}_{t,x^\rho_t,a})^{1/2} }{\delta} \right)   \right) } \label{eqn:confterm1}
\end{align}
for $a \in {\cal A}$ and 
\begin{align}
 & \text{conf}^{\rho}_{t,x^\rho_t,stop} \notag \\
 &=   \sqrt{ \frac{(1 + N^{\rho}_{t,x^\rho_t})}{ (N^{\rho}_{t,x^\rho_t} )^2} 
 \left( 4 \sigma^2 \log \left( \frac{ K (1 + N^{\rho}_{t,x^\rho_t})^{1/2} }{\delta} \right)   \right) } \label{eqn:confterm2}  
\end{align}
where $K = l_{\max} S_{ {\cal X} } S_{ \bar{{\cal A}} }$.
If $stop \in \argmax_{a \in \bar{{\cal A}}} u^{\rho}_{t,x^\rho_t,a}$, then FeedBAL selects the $stop$ action in step $t$. Otherwise, FeedBAL selects one of the actions in ${\cal A}$ with the maximum UCB, i.e., 
$a_t \in \argmax_{a \in {\cal A}} u^{\rho}_{t,x^\rho_t,a}$.     
After selecting the action in step $t$, FeedBAL observes the feedback $f^\rho_t \sim p_{t,x^\rho_t,a_t}$, which is then used to calculate the next state as 
$x^{\rho}_{t+1} = \phi_t(x^{\rho}_{t}, a_t, f^\rho_t)$.    

This procedure repeats until FeedBAL takes the $stop$ action, which will eventually happen since the number of steps is bounded by $l_{\max}$. This way the length of the sequence of selected actions is adapted based on the sequence of received feedbacks and costs of taking the actions. 
After episode $\rho$ ends, FeedBAL observes the costs $C^{\rho}_t$, $ t \in [T_{\rho} - 1]$ and the terminal rewards $R^{\rho}_t$, $ t \in [T_{\rho}]$. 
Finally, using these values, FeedBAL updates the values of the sample mean gains and the counters before episode $\rho+1$ starts (line 16 of FeedBAL), and
reaches its objective of maximizing the expected cumulative gain by capturing the tradeoff between the rewards and the costs of selecting actions. 

%% file: regret.tex
We bound the regret of FeedBAL by bounding the number of times that it will take an action that is different from the action selected by the benchmark. 

Let $g^*_{t,x} = \max_{a \in \bar{{\cal A}}} g_{t,x,a}$ be the gain of the best action and $\Delta_{t,x,a} = g^*_{t,x} - g_{t,x,a}$ be the suboptimality gap of action $a$ for the step-state pair $(t,x)$. The set of optimal actions for step-state pair $(t,x)$ is given by ${\cal O}_{t,x} := \{ a \in \bar{{\cal A}}: \Delta_{t,x,a} =0  \}$. We impose the following assumption in the rest of this section.
%
\begin{assumption}\label{ass:ongains}
For any step-state pair $(t,x)$: (i) $stop \in {\cal O}_{t,x} \Rightarrow  |{\cal O}_{t,x}| =1$, (ii) $|{\cal O}_{t,x}| > 1 \Rightarrow {\cal O}_{t,x} \subset {\cal A}$.
\end{assumption}

Assumption \ref{ass:ongains} implies that ${\cal O}_{t,x}$ cannot include both the $stop$ action and another action in ${\cal A}$. This assumption is required for our regret analysis.
If ${\cal O}_{t,x}$ includes both the $stop$ action and another action in ${\cal A}$, then any learning algorithm may incur linear regret. The reason for this is that the benchmark will always choose the $stop$ action in this case, whereas the learner may take the other action more than it takes the $stop$ action due to the fluctuations of the sample mean gains around their expected values. To circumvent this effect, the learner can add a small positive bias $\epsilon > 0$ to the gain of the $stop$ action. If this bias is small enough such that he $stop$ action remains suboptimal for any step-state pair $(t,x)$ in which the $stop$ action was suboptimal, then our regret analysis can also be applied to the case when Assumption \ref{ass:ongains} is violated. 

Let 
\begin{align}
{\cal E}_{\text{conf}} :=& \left\{ | \hat{g}^{\rho}_{t,x,a} - g_{t,x,a} | \leq c^{\rho}_{t,x,a} \right. \notag \\
&  \left. ~~~~~~  \forall \rho \geq 2, ~ \forall t \in [l_{\max}], ~ \forall  x \in {\cal X}, ~\forall a \in \bar{{\cal A}} \right\} \notag
\end{align}
be the event that the sample mean gains are within $c^{\rho}_{t,x,a}$ of the expected gains. 
The following lemma bounds the probability that ${\cal E}_{\text{conf}}$ happens. 

\begin{lemma} \label{lemma:probabilitybound}
$\Pr({\cal E}_{\text{conf}}) \geq 1 - \delta$.
\end{lemma}
\comment{
\begin{proof}
Fix any step-state-action triplet $(t,x,a)$. 
Let 
\begin{align}
{\cal E}_{\text{conf}}(t,x,a) :=  \left\{ | \hat{g}^{\rho}_{t,x,a} - g_{t,x,a} | \leq c_{t,x,a} ~ \forall \rho \geq 1  \right\}      \notag
\end{align}
By replacing $\delta$ term in \eqref{eqn:stopbound} and \eqref{eqn:sbound5} given in Appendix \ref{app:tripletbound} with $\delta/( l_{\max} S_{ {\cal X} } S_{ \bar{{\cal A}} } )$, we get
$\Pr( {\cal E}_{\text{conf}}(t,x,a) ) \geq 1 - \delta/( l_{\max} S_{ {\cal X} } S_{ \bar{{\cal A}} })$ (details can be found in Appendix \ref{app:tripletbound}). This implies that $\Pr( {\cal E}^c_{\text{conf}}(t,x,a) ) \leq \delta/( l_{\max} S_{ {\cal X} } S_{ \bar{{\cal A}} } )$ for all $t \in [l_{\max}]$, $x \in {\cal X}$ and $a \in \bar{ {\cal A} }$. 
 Using a union bound, we get 
\begin{align}
\Pr({\cal E}^c_{\text{conf}}) 
&= \Pr \left( \bigcup_{t \in [l_{\max}]} \bigcup_{x \in {\cal X}} \bigcup_{a \in \bar{{\cal A}}}  {\cal E}^c_{\text{conf}}(t,x,a) \right)       \notag \\
&\leq \sum_{t \in [l_{\max}]} \sum_{x \in {\cal X}} \sum_{a \in \bar{{\cal A}}} \Pr( {\cal E}^c_{\text{conf}}(t,x,a) ) \notag \\
& \leq \delta. \notag 
\end{align}
\end{proof}
}

 The next lemma upper bounds the number of times each action can be selected on event ${\cal E}_{\text{conf}}$. 
 \begin{lemma} \label{lemma:actioncountbound}
On event ${\cal E}_{\text{conf}}$ we have
 \begin{align}
N^{\rho}_{t,x,a} \leq 3 + \frac{16 \sigma^2}{\Delta^2_{t,x,a}} \log( \frac{16 \sigma^2 K} { \Delta^2_{t,x,a} \delta   }  )  ~~   & \forall \rho \geq 1, ~\forall  t \in [l_{\max}], \notag \\ 
& \forall x \in {\cal X}, ~\forall a \in \bar{{\cal A}} .     \notag
 \end{align}
 \end{lemma} 
\comment{
 \begin{proof}
For $\rho=1$, the result is trivial. For $\rho > 1$, the proof proceeds in a way that is similar to the proof of Lemma 6 in \cite{abbasi2011improved}. Firstly, assume that action $a \in \bar{{\cal A}}$ is selected in step $t \in [T_{\rho}]$ of episode $\rho$ when the state is $x$.
Since 
\begin{align}
\hat{g}^{\rho}_{t,x,a} &\in [ g_{t,x,a} - \text{conf}^{\rho}_{t,x,a},  g_{t,x,a} + \text{conf}^{\rho}_{t,x,a}  ] \notag \\ 
 \hat{g}^{\rho}_{t,x,a^*} &\in [ g^*_{t,x} - \text{conf}^{\rho}_{t,x,a^*},  g^*_{t,x} + \text{conf}^{\rho}_{t,x,a^*}  ], ~~ a^* \in {\cal O}_{t,x} \notag
\end{align}
on event ${\cal E}_{\text{conf}}$, using
\begin{align}
\hat{g}^{\rho}_{t,x,a} + \text{conf}^{\rho}_{t,x,a} & \geq  g^*_{t,x}  \label{eqn:confi1} \\
\hat{g}^{\rho}_{t,x,a}  & \leq g_{t,x,a} + \text{conf}^{\rho}_{t,x,a} \label{eqn:confi2}
\end{align}
and the definition of $\Delta_{t,x,a}$, we obtain $\text{conf}^{\rho}_{t,x,a} \geq \Delta_{t,x,a} / 2$. Substituting the values in \eqref{eqn:confterm1} and \eqref{eqn:confterm2} into $\text{conf}^{\rho}_{t,x,a}$ and using the fact that $(z^2 - 1)/(z+1) \leq z^2/(z+1)$ for positive integers $z$, we get for $a \in {\cal A}$
\begin{align}
\frac{ (N^{\rho}_{t,x,a})^2 - 1 }{ N^{\rho}_{t,x,a} +1  } 
\leq \frac{4}{\Delta^2_{t,x,a}}
 \left( 4 \sigma^2 \log \left( \frac{ K (1 + N^{\rho}_{t,x,a})^{1/2} }{\delta} \right)   \right) \label{eqn:trickterm}
\end{align}

Finally, assume that the $s := stop$ action is selected in step $t =T_{\rho}$ of episode $\rho$ when the state is $x$. Let 
\begin{align}
\overline{\text{conf}}^{\rho}_{t,x,s} =  \sqrt{ \frac{(1 + N^{\rho}_{t,x,s})}{ (N^{\rho}_{t,x,s} )^2} 
\left( 4 \sigma^2 \log \left( \frac{ K (1 + N^{\rho}_{t,x,s})^{1/2} }{\delta} \right)   \right) }  .  \notag
\end{align}
Since $N^{\rho}_{t,x,s} \leq N^{\rho}_{t,x}$, we have $\text{conf}^{\rho}_{t,x,s} \leq \overline{\text{conf}}^{\rho}_{t,x,s}$, which implies that on event ${\cal E}_{\text{conf}}$
\begin{align}
\hat{g}^{\rho}_{t,x,a} &\in [ g_{t,x,a} - \overline{\text{conf}}^{\rho}_{t,x,s},  g_{t,x,a} + \overline{\text{conf}}^{\rho}_{t,x,s} ]      \notag \\
 \hat{g}^{\rho}_{t,x,a^*} &\in [ g^*_{t,x} - \text{conf}^{\rho}_{t,x,a^*},  g^*_{t,x} + \text{conf}^{\rho}_{t,x,a^*}  ], ~~ a^* \in {\cal O}_{t,x} \notag
\end{align}
Using
\begin{align}
\hat{g}^{\rho}_{t,x,s} + \overline{\text{conf}^{\rho}_{t,x,s}} & \geq  g^*_{t,x}  \notag \\
\hat{g}^{\rho}_{t,x,s}  & \leq g_{t,x,s} + \overline{\text{conf}^{\rho}_{t,x,a}} \notag
\end{align}
and the definition of $\Delta_{t,x,a}$, we obtain $\overline{\text{conf}^{\rho}_{t,x,s}} \geq \Delta_{t,x,s} / 2$. This implies that \eqref{eqn:trickterm} also holds for $a = stop$.

Next, we use the a lemma from \cite{antos2010active} to bound \eqref{eqn:trickterm}, which also also given in Appendix \ref{app:antos}. 
From \eqref{eqn:trickterm} we obtain
\begin{align}
N^{\rho}_{t,x,a} \leq 1 + \frac{16\sigma^2}{\Delta^2_{t,x,a}} \log \left( \frac{K}{\delta} \right)  
+ \frac{8\sigma^2}{\Delta^2_{t,x,a}} \log (1 + N^{\rho}_{t,x,a} )   \label{eqn:trickterm2}
\end{align}
Since $1 + N^{\rho}_{t,x,a}   \geq 1$, we substitute $a= \Delta^2_{t,x,a} / (16 \sigma^2 )$ and $b = \log (   16 \sigma^2 / \Delta^2_{t,x,a})$ in Appendix \ref{app:antos} to get the bound
\begin{align}
\log (1 + N^{\rho}_{t,x,a} ) \leq a (1 + N^{\rho}_{t,x,a} ) + b   \notag
\end{align}
The result is obtained by substituting this into \eqref{eqn:trickterm2}.
\end{proof}
}

 As a corollary of Lemma \ref{lemma:actioncountbound} we derive the following bound on the confidence of the actions selected by FeedBAL.
 \begin{corollary} \label{cor:confidence}
 With probability at least $1-\delta$
 \begin{align}
\forall \rho \geq 2, \forall t \in [l_{\max}]   
~~ g^*_{t,x^{\rho}_t}  - g_{t,x^{\rho}_t, a^{\rho}_t}  \leq 2 \text{conf}^{\rho}_{t,x^{\rho}_t,a^{\rho}_t} .\notag 
 \end{align}
 \end{corollary}
 \comment{
 \begin{proof}
The result follows by a simple application of \eqref{eqn:confi1} and \eqref{eqn:confi2} on event ${\cal E}_{\text{conf}}$.
 \end{proof}
}
 
Corollary \ref{cor:confidence} bounds the suboptimality of the action selected by FeedBAL in any step of any episode by $2 \text{conf}^{\rho}_{t,x^{\rho}_t,a^{\rho}_t}$, which only depends on quantities $\delta$, $K$, $\sigma^2$ and $N^{\rho}_{t,x^{\rho}_t, a^{\rho}_t}$, which are known by the learner at the time $a^{\rho}_t$ is selected. 

Consider any algorithm that deviates from the benchmark for the first time in step-state pair $(t,x)$ by choosing action $a$ that is different from the action that will be chosen by the benchmark at $(t,x)$. 
Let $\mu^*_{t,x}$ be the maximum expected gain that can be acquired by the benchmark starting from step-state pair $(t,x)$.\footnote{In calculating $\mu^*_{t,x}$, we assume that in steps in which the benchmark needs to randomize between at least two actions, the action that maximizes the expected reward of the benchmark is selected.} Let $\underline{\mu}_{t,x,a}$ be the minimum expected gain that can be acquired by any algorithm by choosing the worst-sequence of actions starting from step-state pair $(t,x)$ after chosing action $a$. We define the {\em deviation gap} in step-state pair $(t,x)$ as $\Omega_{t,x,a} := \mu^*_{t,x} - \underline{\mu}_{t,x,a}$. The following theorem show that the regret of FeedBAL is bounded with probability at least $1-\delta$.   

\begin{theorem} \label{thm:boundedregret}
With probability at least $1-\delta$, the regret of FeedBAL given in \eqref{eqn:pseudoregret} is bounded by
\begin{align}
R(n) \leq \sum_{t =1}^{l_{\max}} \sum_{x \in {\cal X}} \sum_{a \notin {\cal O}_{t,x} } \Omega_{t,x,a}
\left( 3 + \frac{16 \sigma^2}{\Delta^2_{t,x,a}} \log( \frac{16 \sigma^2 K} { \Delta^2_{t,x,a} \delta   }  )  \right) \notag
\end{align}
\end{theorem}
\comment{
\begin{proof}
The proof directly follows by summing the result of Lemma \ref{lemma:actioncountbound} among all step-state-action triplets $(t,x,a)$.
\end{proof}
}

The bound given in Theorem \ref{thm:boundedregret} does not depend on $n$. As given in the following theorem, this bound can be easily converted to a bound on the expected regret by setting $\delta = 1/n$.

\begin{theorem} \label{thm:logregret}
When FeedBAL is run with $\delta = 1/n$, its expected regret given in \eqref{eqn:regretdef} is bounded by 
\begin{align}
&\expect{ R(n) }  \leq  \Omega_{\max}   \notag \\
&+ \sum_{t =1}^{l_{\max}} \sum_{x \in {\cal X}} \sum_{a \notin {\cal O}_{t,x} } \Omega_{t,x,a}
\left( 3 + \frac{16 \sigma^2}{\Delta^2_{t,x,a}} \log( \frac{16 \sigma^2 K n} { \Delta^2_{t,x,a}   } )  \right) \notag
\end{align}
where $\Omega_{\max} = \max_{t,x,a} \Omega(t,x,a)$.
\end{theorem}
\comment{
\begin{proof}
Consider Theorem \ref{thm:boundedregret}. With probability $\delta$, the regret is bounded above by $n \Omega_{\max}$. With probability $1-\delta$, the regret is bounded by this theorem's main statement. The result is obtained by applying the law of total expectation.
\end{proof}
}

Theorem \ref{thm:logregret} shows that the expected regret of FeedBAL is $O(\log n)$. Although the constant terms given in Theorems \ref{thm:boundedregret} and \ref{thm:logregret} depend on unknown parameters $\Delta_{t,x,a}$ and $\Omega_{t,x,a}$, FeedBAL does not require the knowledge of these parameters to run and to calculate its confidence bounds. From the expressions in Theorems \ref{thm:boundedregret} and \ref{thm:logregret}, it is observed that the regret scales linearly with $\Omega_{t,x,a} / \Delta^2_{t,x,a}$, which is a term that indicates the {\em hardness} of the problem. If the suboptimality gap $\Delta^2_{t,x,a}$ is small, FeedBAL makes more errors by choosing $a \notin {\cal O}_{t,x}$ when it tries to follow the benchmark. This results in a loss in the expected gain that is bounded by $\Omega_{t,x,a}$. 

Next, we consider problems in which deviations from the benchmark in early steps cost more than deviations from the benchmark at later steps. 
\begin{assumption}\label{ass:earlydeviations}
$\Omega_{t,x,a} \leq (l_{\max} - t) \Delta_{t,x,a}$ for all $t \in [l_{\max}]$, $x \in {\cal X}$, $a \notin {\cal O}_{t,x}$. 
\end{assumption}
Using this assumption, the following result is derived for the expected regret of FeedBAL.

\begin{corollary} \label{cor:diminishinglosses}
When Assumption 2 holds, and FeedBAL is run with $\delta = 1 /n$, we have
\begin{align}
& \expect{ R(n) }  \leq  \Omega_{\max}   \notag \\
&+ l_{\max} \sum_{t =1}^{l_{\max}}  \sum_{x \in {\cal X}} \sum_{a \notin {\cal O}_{t,x} } 
\left( 3 \Delta_{t,x,a} + \frac{16 \sigma^2}{\Delta_{t,x,a}} \log( \frac{16 \sigma^2 K n} { \Delta^2_{t,x,a}   } )  \right) \notag
\end{align}
\end{corollary}
Although, the regret bound of FeedBAL increases polynomially in the size of the state-space, for many interesting applications of eMAB, the state-space is small. For instance, consider the breast cancer treatment example in \citet{pardalos2009handbook}. In this example, ${\cal X}$ has only four states: no cancer, in suti cancer, invasive ductal carcinoma, dead.\footnote{The reward assigned to state ``dead" can be $0$, and to state ``no cancer" can be $1$.} ${\cal A}$ is the set of treatment options, and ${\cal F}$ is the feedback set, which can be the reduction in tumor size given a particular treatment in a particular state. 

\begin{remark}
FeedBAL adaptively learns the expected gains of action and feedback sequences that correspond to stopping at various steps. Although our model allows at most $l_{\max}$ actions to be taken in each episode, the actual number of actions taken may be much lower than this value depending on the expected costs $c_{t,x,a}$. High costs implies a decrease in the marginal benefit of continuation, which implies that the benchmark may take the stop action earlier than the case when costs are low. 
\end{remark}
\begin{remark}
The state-space model we proposed is very general, and as we stated in Section III-B, includes the adaptive monotone submodular problem \cite{golovin2010adaptive, gabillon2013adaptive} as a special case. The state-space model of eMAB generalizes these problems in a way that the distribution of feedback given the action also depends on the state of the system. 
\end{remark}

%% file: experiments2.tex
\textbf{Actions, feedbacks, states, rewards and costs}

We consider a game where the learner aims at collecting resources to maximize its payoff, where the payoff depends both on the number of collected resources and the duration of resource collection. Let $X_{t}$ denote the binary-valued random variable which takes value $1$ if resource is present in step $t$ and $0$ otherwise. At the beginning of each episode $X_t$ is drawn from a Bernoulli distribution with parameter $p_t$ independently from the other steps. $[X_1,\ldots,X_{l_{\max}-1}]^T$ represents the resource vector. We assume that $p_t > p_{t+1}$, $\forall t \in [l_{\max}-2]$ to model a decaying resource generation rate. 

The learner has only two actions: $cont$ and $stop$. When the learner takes $cont$ action in step $t$, it moves to the next step, observes as feedback $X_{t}$ and pays cost $\eta_t$, where the expected cost is $0$ and $\eta_t$ is zero mean Gaussian noise with variance \revn{$\sigma^2_c$}. As usual, $stop$ ends the episode. The state space is ${\cal X} = \{0,\ldots,l_{\max}-1\}$, and the state at step $t$ is $N_t$, which is the number of resources collected by the beginning of step $t$. Thus, $\phi(t,x,cont) = x + X_{t}$.
The benefit that the learner obtains from collected resources exponentially decays with time. Hence, the terminal reward of stopping at step-state pair $(t,x)$ is set as $\beta^{t-1} x + \kappa_t$, where $\kappa_t$ is zero mean Gaussian noise with variance \revn{$\sigma^2_r$} and $\beta \in (0,1]$ is the discount factor. 

Next, we prove that the benchmark is optimal. Since the expected total cost is zero, the expected cumulative gain is equal to the expected terminal reward. 
We have $r_{t,x} = \beta^{t-1}x$ and $y_{t,x,cont} = \beta^t(x+p_t)$. Obviously, it is not optimal to stop when the benchmark selects $cont$ at step-state pair $(t,x)$ since instead of stopping, continuing for one more step, and then stopping yields ex-ante terminal reward $y_{t,x,cont} > r_{t,x}$. We also show that continuing when the benchmark selects $stop$ always yields an expected terminal reward that is less than or equal to the expected terminal reward of the benchmark. For this, consider the case that the benchmark stops at step-state pair $(t,x)$, which implies that, $x \geq \beta p_t / (1-\beta)$. This implies that for any $y \geq x$ and $j \geq 0$, $y \geq \beta p_{t+j} / (1-\beta)$. Let $y_0 = x$, and $y_i = y_{i-1} + p_{t+i-1}$ for $i=1,\ldots,k$. Thus, we have $y_i \geq \beta p_{t+i}/(1-\beta)$, which implies that 
\begin{align}
\beta^{t+i-1} y_i \geq \beta^{t+i} y_{i+1}. \label{eqn:they}	
\end{align}
Clearly, stopping $k$ steps after $t$ yields expected terminal reward $\beta^{t+k-1} (x+p_t+\ldots+p_{t+k-1})$. Using \eqref{eqn:they}, we obtain $r_{t,x} \geq \beta^{t} (x+p_t) \geq \beta^{t+1} (x+p_t + p_{t+1}) \geq \ldots \geq \beta^{t+k-1} (x+p_t+\ldots+p_{t+k-1})$, which implies that the benchmark is optimal. 

\textbf{Results}

We compare FeedBAL with two algorithms. The first one is UCB1 \cite{auer}, whose arms are sequences of actions with maximum length $l_{\max}$, where only the last action is $stop$. UCB1 chooses an arm at the beginning of each episode, selects actions according to the chosen arm, and updates only the empirical cumulative gain of the chosen arm at the end of the episode. The second one is a variant of UCB1, which we call UCB1-V. At the end of each episode, UCB1-V updates the empirical cumulative gains of all arms whose terminal rewards are observed (we call these updates virtual updates). For instance, if UCB1-V chooses the arm that corresponds to the sequence with $l$ $cont$ actions followed by $stop$, then it updates the empirical cumulative gains of all arms that correspond to the sequences with $j$ $cont$ actions followed by $stop$ for all $j \leq l$.

\begin{figure}[h!]
	\centering
	\includegraphics[width=0.9\columnwidth]{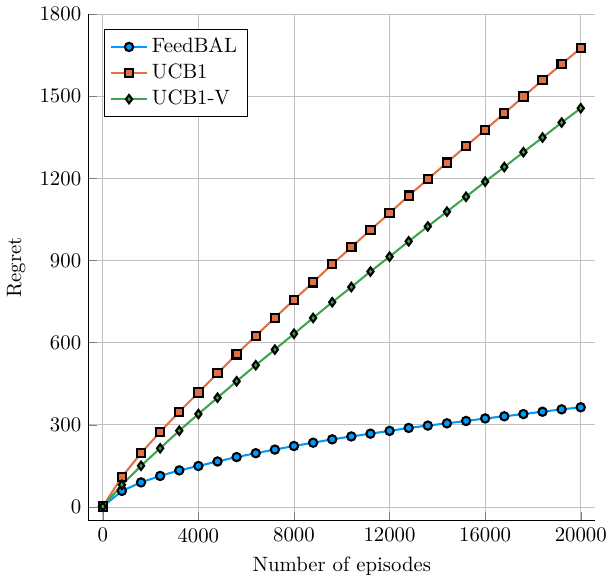}
	\vspace{-0.2in}
	\caption{Regrets of FeedBAL, UCB1 and UCB1-V as a function of the number of episodes.}
	\label{fig_result}
	\vspace{-0.3in}
\end{figure}

In simulations, we set \revn{$l_{\max}=10$, $p_t = 0.8 / \sqrt{t}$, $\beta=0.9$, $\sigma^2_c=0.1$, $\sigma^2_r=0.1$, $\delta=0.01$, $\sigma = \sqrt{0.2}$, and $n=20000$.} Since the benchmark is optimal, we plot the regrets of all algorithms averaged over $1000$ runs against the benchmark. From the results given in Figure \ref{fig_result}, we observe that FeedBAL incurs very small regret and significantly outperforms UCB1 and UCB1-V. The superior performance of FeedBAL comes from the fact that it is able to adapt the action selections based on the feedbacks observed during an episode.

%% file: related.tex

eMAB is related to various existing classes of MAB with large action
sets. These include combinatorial bandits \cite{cesa2012combinatorial,gai2012combinatorial},
combinatorial semi-bandits \cite{kveton2015tight}, matroid bandits
\cite{kveton2014matroid}, and bandits in metric spaces \cite{kleinberg2008multi}.
In these works, at each time, the learner (simultaneously) chooses an action tuple and obtains
a reward that is a function of the chosen action tuple. Unlike these works, in eMAB actions in an episode are chosen sequentially, and the previously chosen actions in an episode guide the action selection process within that episode.  

\comment{
The differences between eMAB and these
various classes of MAB are given in Table \ref{table:compcomb}.
\begin{table}[h!]
\centering
\caption{Comparison of eMAB with combinatorial, semi and matroid bandits.}
\label{table:compcomb}
{\fontsize{10}{10}\selectfont
\setlength{\tabcolsep}{.1em}
\begin{tabular}{|l|c|c|}
\hline
         &  \cite{cesa2012combinatorial, gai2010learning, gai2012combinatorial, kveton2015tight, kveton2014matroid} & eMAB \\
         \hline
Arm selection & multiple - simultaneous &  multiple - sequential \\
in each episode &										   &                                     \\
\hline
Reward in    & sum of rewards    & general function of  \\
each episode  & of selected arms  & selected arms and \\
                   &                             & observed feedbacks \\
\hline
Action sequence  & fixed and limited & variable, not limited  \\
length & by action set size & by action set size \\ 
\hline
\end{tabular}
}
\end{table}
}


Another related strand of literature studies MAB with knapsacks \cite{badanidiyuru2013bandits,tran2012knapsack}.
In these problems, there is a {\em budget}, which limits the number
of times a particular action can be selected. The goal is to maximize
the total reward given the budget constraints. However, similar to
standard MAB problems, in these problems it is also assumed that the
reward is immediately available after each selected action, and the
current reward only depends on the current action \ck{unlike eMAB in
which the current reward depends on a sequence of actions and feedbacks through a state.}
Moreover, in eMAB, the budget is renewed after
each episode; and hence, does not limit the number of episodes in which
a certain action can be selected as in MAB with knapsacks.


One of the most closely related prior works is the work on adaptive submodularity
\cite{golovin2010adaptive} where it is shown
that for adaptive submodular reward functions, a simple adaptive greedy policy 
(which resembles our benchmark) is $1-1/e$
approximately optimal. Hence, any learning algorithm that has sublinear
regret with respect to the greedy policy is guaranteed to be approximately
optimal. This work is extended to an online setting in \citet{gabillon2013adaptive},
where prior distribution over the state is unknown and only the reward
of the chosen sequence of actions is observed. However, an independence
assumption is imposed over action states to estimate the prior in
a fast manner. In these works the goal is to select the optimal sequence of items or actions (without replacement) given a fixed budget (on the number of steps), and the item states (feedbacks) are realized before the episode begins.
On the other hand, in our formulation, the same action can be taken in different steps, the number of steps is not fixed but is adapted based on the feedback, and feedback in the current step depends on actions
and feedbacks in prior steps of the current episode.



\comment{
	In this work, existence of an underlying
	joint state that is realized from a prior distribution before the
	start of the action selection process in each episode is assumed. 
	The learner is constrained to pick sequentially $K$ of the actions in the action set in each episode
	without replacement in order to maximize its reward. The states of the selected
	actions are instantly revealed to the learner; hence, at decision step $t$ of an episode, the learner knows the states of all arms selected before step $t$. The reward the learner gets at the end of an episode is a submodular function of the selected arms and their states. 
	
	If we translate the above setting to our eMAB formulation, 
	the joint state can be viewed as a hidden state vector for actions,
	whose components are revealed only after the corresponding actions
	are taken. Hence, this is a special case of the eMAB, where the joint state 
	does not depend on the chosen actions and observed feedbacks. 
}


Our problem is also related to reinforcement learning in MDPs. For instance, in \citet{tewari2008} and \citet{auer2009near}
algorithms with logarithmic regret with respect to the optimal policy
are derived for finite, recurrent MDPs. However, the proposed algorithms rely on variants of value iteration or linear programming, and hence, have higher computational complexity than our proposed method.
Episodic MDPs are
studied in \citet{zimin2013online}, and sublinear regret bounds are
derived assuming that the loss sequence is generated by an adversary.
eMAB differs from these works as follows: 
(i) the number of visited states (steps) in each episode is not fixed; 
(ii) During an episode, only feedbacks are observed and no reward observations are available for the intermediate states; (iii) Rewards of the intermediate states are only revealed at the end of the episode.
Recently, improved gap-independent regret bounds are derived for reinforcement learning in MDPs by using an optimistic version of value iteration \cite{azar2017minimax} for episodic MDPs and posterior sampling for non-episodic MDPs \cite{agrawal2017optimistic}.
While it is possible to translate eMAB into an MDP, finding the optimal policy in the MDP is more challenging than competing with our benchmark, both in terms of the speed of learning and cost of computation.
Thus, eMAB can be seen as a bridge between standard MAB and reinforcement learning in MDPs, where the order of actions taken in each episode matters and the learner aims to perform as good as a {\em moderate} benchmark which may not always be optimal, but outperforms the best fixed action and works well in a wide range of settings.  

\comment{
\begin{table}[h!]
\centering
\caption{Comparison of eMAB with optimization and reinforcement learning algorithms.}
\label{table:comprein}
{\fontsize{10}{10}\selectfont
\setlength{\tabcolsep}{.1em}
\begin{tabular}{|l|c|c|c|}
\hline
& PI, VI  & Q-learning, TD($\lambda$) & eMAB \\
\hline
Transition  & known &  unknown & unknown \\
probabilities & & & \\
\hline
Convergence & always  & may converge  & converges \\
to optimal   &    optimal  & asymptotically & asymptotically  \\
\hline
Regret    & zero    & may be      & logarithmic   \\ 
in time   &  & sublinear   &                      \\      
\hline
Efficient for: & small action  & small action & large action  \\
& sequences & sequences & sequences \\
\hline
\end{tabular}
}
\end{table}
}

\comment{
\subsection{Online and Stochastic Convex Optimization}

We will finish our discussion of the related work by differentiating
eMAB from {\em Online Convex Optimization} (OCO) and 
{\em Stochastic Convex Optimization} (SCO) problems.

In OCO, there is a learner which sequentially chooses actions (from a
convex set) over time, and incurs a loss (from a convex function)
after each chosen action. The loss function is generated by an adversary
and is unknown to the learner beforehand. The goal of the learner
is to minimize its regret, which is the difference between the total
loss it accumulates and the loss of the best fixed action it could
have followed (best fixed strategy in hindsight). Many versions of
OCO exist including {\em full feedback} \cite{kivinen1997exponentiated,zinkevich2003online,hazan2007logarithmic},
in which the learner observes the entire loss function after each
decision step, and {\em bandit feedback} \cite{flaxman2005online,awerbuch2008online},
in which the learner partially observes the loss evaluated at the
chosen action.

OCO and eMAB have two fundamental differences:
(i) In eMAB, the learner selects multiple actions during an episode, and the selected
actions effect the actions that will be selected in future; whereas
in OCO full or partial loss function is observed after every taken
action, and the reward only depends on the current taken action. 
(ii) The regret of an eMAB algorithm is measured with respect to the adaptive
benchmark, which myopically adapts the next action to select in an
episode based on the previously selected actions and observed feedbacks;
whereas in OCO the regret is measured with respect to the best fixed
action in hindsight. While the action sequence selected by the adaptive
benchmark can change from episode to episode based on the sequence of
observed feedbacks, the action sequence selected by the benchmark
of the OCO is fixed among episodes. Hence eMAB and OCO are different
both in terms of the way rewards are generated and performance is
evaluated.

In SCO, the goal is to minimize a convex loss function using finite number queries obtained from a gradient oracle \cite{hazan2014beyond}. Numerous methods have been proposed to efficiently solve this problem, such as a batch reduction from an OCO problem and {\em alternating direction method of
multipliers} based methods \cite{tao2014stochastic}. The differences of SCO from eMAB are similar in flavor to that of OCO from eMAB. In addition to this, the objective function of SCO is also different from that of OCO and eMAB.
}

%% file: tripletbound.tex
First, we consider the confidence bound for the $stop$ action. Fix $t \in [l_{\max}]$ and $x \in {\cal X}$. Let $\epsilon_{\rho} = \mr{I} (t \leq T_{\rho}, x^{\rho}_t = x)$. \revn{Since, $\{ \kappa^{\rho}_t \}_{\rho=1}^{\infty}$ is a sequence of $\sigma$-sub-Gaussian random variables,} using the result of Theorem 1 in \citet{abbasi2011improved}, it can be shown that given any $\delta >0$ with probability at least $1- \delta$ we have for all $\rho \geq 2$
\begin{align}
& \left( \frac{ | \sum_{l=1}^{\rho-1} \epsilon_{l} \kappa^{l}_t | }{ \sqrt{1 + N^{\rho}_{t,x} } } \right)^2  
 \leq 2 \sigma^2  \log \left( \frac{ \sqrt{1 + N^{\rho}_{t,x}} } {\delta}      \right)   \Rightarrow    \notag \\
&
 \bigg| \sum_{l=1}^{\rho-1} \epsilon_{l} \kappa^{l}_t \bigg| 
\leq \sqrt{ (1 + N^{\rho}_{t,x})  2 \sigma^2  \log \left( \frac{ \sqrt{1 + N^{\rho}_{t,x}} } {\delta}      \right)    } . \label{eqn:sbound3}
\end{align}
Observe that 
\begin{align}
\hat{r}^{\rho}_{t,x} = \frac{\sum_{l=1}^{\rho-1} ( r_{t,x} \epsilon_{l} + \kappa^{l}_t \epsilon_{l} )}{ N^{\rho}_{t,x}  } = r_{t,x} + \frac{\sum_{l=1}^{\rho-1}  \kappa^{l}_t \epsilon_{l} } { N^{\rho}_{t,x}  } .      \notag
\end{align}
Hence
\begin{align}
| \hat{r}^{\rho}_{t,x} - r_{t,x} | = \frac{1}{N^{\rho}_{t,x}}  \bigg| \sum_{l=1}^{\rho-1} \epsilon_{l} \kappa^{l}_t \bigg| .    \notag
\end{align}
Combining this with \eqref{eqn:sbound3} we obtain with probability at least $1-\delta$ for all $\rho \geq 2$
\begin{align}
| \hat{r}^{\rho}_{t,x} - r_{t,x} | 
\leq 
\sqrt{ \frac{(1 + N^{\rho}_{t,x})}{(N^{\rho}_{t,x})^2} 2 \sigma^2  \log \left( \frac{ \sqrt{1 + N^{\rho}_{t,x}} } {\delta}      \right)   }  .
      \notag
\end{align}
Since by definition $g_{t,x,stop} = r_{t,x}$ and $\hat{g}_{t,x,stop} = \hat{r}_{t,x}$ we get with probability at least $1-\delta$
\begin{align}
\forall \rho \geq 2 ~~ & | \hat{g}^{\rho}_{t,x,stop} - g_{t,x,stop} | \notag \\
&\leq \sqrt{ \frac{(1 + N^{\rho}_{t,x})}{(N^{\rho}_{t,x})^2} 2 \sigma^2  \log \left( \frac{ \sqrt{1 + N^{\rho}_{t,x}} } {\delta}      \right) } .  \label{eqn:stopbound}
\end{align}     

Next, we consider the confidence bound for actions $a \in {\cal A}$. Fix $t \in [l_{\max}]$, $x \in {\cal X}$ and $a \in {\cal A}$.
With an abuse of notation let $\epsilon_{\rho} = \mr{I} (t < T_{\rho}, x^{\rho}_t = x, a^{\rho}_t = a   )$. Consider the random variable 
\begin{align}
Y^{\rho}_t := R^{\rho}_{t+1} - C^{\rho}_t = g_{t,x^\rho_t,a^\rho_t} + \kappa^{\rho}_{t+1} -\eta^{\rho}_t      \notag
\end{align}
which is used to update the sample mean gain (\rev{line 16} of \revn{Algorithm 2}). 
Let $\beta^{\rho}_t := \kappa^{\rho}_{t+1} -\eta^{\rho}_t$. 
Since $\kappa^{\rho}_{t+1}$ and $\eta^{\rho}_t$ are independent $\sigma$-sub-Gaussian random variables, $\beta^{\rho}_t$ is $\sqrt{2} \sigma$-sub-Gaussian. In addition, $\{ \beta^{\rho}_t \}_{\rho=1}^{\infty}$ is a sequence of independent random variables. 

Using the result of Theorem 1 in \citet{abbasi2011improved}, it can be shown that given any $\delta >0$ with probability at least $1- \delta$ we have for all $\rho \geq 2$
\begin{align}
& \left( \frac{ | \sum_{l=1}^{\rho-1} \epsilon_{l} \beta^{l}_t | }{ \sqrt{1 + N^{\rho}_{t,x,a} } } \right)^2  
 \leq 4 \sigma^2  \log \left( \frac{ \sqrt{1 + N^{\rho}_{t,x,a}} } {\delta}      \right)     \Rightarrow   \notag \\
&\bigg| \sum_{l=1}^{\rho-1} \epsilon_{l} \beta^{l}_t \bigg| 
\leq \sqrt{ (1 + N^{\rho}_{t,x,a})  4 \sigma^2  \log \left( \frac{ \sqrt{1 + N^{\rho}_{t,x,a}} } {\delta}      \right)    }  . \label{eqn:sbound4}
\end{align}
Observe that 
\begin{align}
\hat{g}^{\rho}_{t,x,a} = \frac{\sum_{l=1}^{\rho-1} ( g_{t,x,a} \epsilon_{l} + \beta^{l}_t \epsilon_{l} )}{ N^{\rho}_{t,x,a}  } = g_{t,x,a} + \frac{\sum_{l=1}^{\rho-1}  \beta^{l}_t \epsilon_{l} } { N^{\rho}_{t,x,a}  } .      \notag
\end{align}
Hence
\begin{align}
| \hat{g}^{\rho}_{t,x,a} - g_{t,x,a} | = \frac{1}{N^{\rho}_{t,x,a}}  \bigg| \sum_{l=1}^{\rho-1} \epsilon_{l} \beta^{l}_t \bigg| .    \notag
\end{align}
Combining this with \eqref{eqn:sbound4} we obtain with probability at least $1-\delta$
\begin{align}
\forall & \rho \geq 2 ~~ 
| \hat{g}^{\rho}_{t,x,a} - g_{t,x,a} | \notag \\
& \leq 
\sqrt{ \frac{(1 + N^{\rho}_{t,x,a})}{(N^{\rho}_{t,x,a})^2} 4 \sigma^2
  \log \left( \frac{ \sqrt{1 + N^{\rho}_{t,x,a}} } {\delta}      \right)   } . \label{eqn:sbound5}
\end{align}

%% file: antos.tex
Let $a > 0$. For any 
\begin{align}
\tau > \frac{2}{a} \left( \log \left( \frac{1}{a} \right) - b \right)^{+}      \notag
\end{align}
we have $a \tau + b \geq \log \tau$, where $a^{+} = \max(a,0)$.